\documentclass[preprint,12pt]{elsarticle}

\usepackage{algorithm2e}
\usepackage{adjustbox}
\usepackage{booktabs}
\usepackage[all]{xy}
\usepackage{graphicx}
\usepackage{longtable}
\usepackage{float}
\usepackage[colorlinks=false]{hyperref}
\usepackage{tikz}

\usepackage{xcolor}

\usepackage{amssymb}
\usepackage{amsmath}
\usepackage{amsthm}

\newtheorem{theorem}{Theorem}[section]
\newtheorem{lemma}[theorem]{Lemma}

\newtheorem{remark}[theorem]{Remark}

\newtheorem{proposition}[theorem]{Proposition}

\newtheorem{definition}[theorem]{Definition}

\usepackage{color}

\begin{document}
	

\begin{frontmatter}



\title{Metric spaces of walks and Lipschitz duality on graphs}


\author{R. Arnau, A. González Cortés, E.A. Sánchez Pérez and S. Sanjuan} 

\affiliation{organization={Instituto Universitario de Matemática Pura y Aplicada,
	Universitat Politècnica de València},
            addressline={Camino de Vera s/n}, 
            city={Valencia},
            postcode={46022}, 
            state={Comunitat Valenciana},
            country={Spain}}

\begin{abstract}
We study the metric structure of walks on graphs, understood as Lipschitz sequences. To this end, a weighted metric is introduced to handle sequences, enabling the definition of distances between walks based on stepwise vertex distances and weighted norms. We analyze the main properties of these metric spaces, which provides the foundation for the analysis of weaker forms of instruments to measure relative distances between walks: proximities. We provide some representation formulas for such proximities under different assumptions and provide explicit constructions for these cases. The resulting metric framework allows the use of classical tools from metric modeling, such as the extension of Lipschitz functions from subspaces of walks, which permits extending proximity functions while preserving fundamental properties via the mentioned representations. Potential applications include the estimation of proximities and the development of reinforcement learning strategies based on exploratory walks, offering a robust approach to Lipschitz regression on network structures.
\end{abstract}


\end{frontmatter}




\section{Introduction}

Graphs serve as natural models for a wide range of network-based systems across various disciplines, such as molecular biology \cite{liu}, economics \cite{zhou} and social media analysis \cite{wang}. Within these contexts, graphs can be effectively understood as significant instances of metric spaces. By endowing the set of vertices of a graph with an appropriate metric (usually related to the general notion of shortest-path distance) we enable the application of geometric and functional analytic techniques traditionally reserved for continuous spaces (see \cite{bani, bour, weaver2018lipschitz}). For example, viewing graphs as metric spaces allows the use of well-studied techniques to embed metric spaces into Banach spaces \cite{weaver2018lipschitz}. 

Outside of classical analysis, graph embeddings have recently gained substantial interest in applied mathematics due to their ability to represent graph-structured data in computational settings \cite{cai,maka}. Such embeddings transform nodes, edges, or entire graphs into finite-dimensional vector spaces while preserving essential structural features of the original graphs. Their importance lies in simplifying complex relationships, enabling more efficient and robust computational methodologies, particularly in machine learning, data analysis, and visualization tasks \cite{dasg,hjal,pileh}. 
In the same direction, graphs understood as metric spaces, as well as Lipschitz functions defined on them, have proven to be useful tools in machine learning \cite{cal,cha,falsan,kyng}.

Following the same trend, this paper explores the application of functional analytic techniques from a different perspective. We focus specifically on the metric structure of spaces of walks on graphs. In our case, the connection to classical analysis arises from the conceptualization of walks as Lipschitz sequences \cite{goy}, which opens the door to the application of powerful analytical tools for theoretical and applied investigations. We propose viewing the space of walks as a dual object in relation to the graph itself, inspired by standard duality theory associated with metric spaces \cite{A.E56}. Our approach introduces a weighted metric on the sequences representing walks, defining distances through vertex-by-vertex comparisons, and a norm derived from a suitable weighted scheme.

We explore the fundamental properties of these newly defined metric spaces, establishing the framework for defining and analyzing functions that quantify relative similarities between walks, which we term proximities. These functions assign a non-negative real value to each pair of walks and are intended to model notions of distance without satisfying all the properties of a metric, thus capturing the widespread use of metric-type functions that are not proper metrics in machine learning (e.g., cosine similarity). We provide explicit representation formulas for these proximities under various scenarios and present concrete constructions. Building on this robust metric framework, we apply classical tools from metric modeling, such as the extension of Lipschitz functions from subspaces, which enables the generalization of proximity functions while preserving their essential properties through established representation techniques.

Potential applications of our framework include estimating proximities and developing reinforcement learning algorithms that rely on exploratory walks, thus offering a powerful methodological basis for Lipschitz regression within network-based structures. Our results provide theoretical insights as well as practical tools for analyzing complex network dynamics effectively.

We will show the results of our study through several sections. After the Introduction, we show in Section \ref{s2} the main results concerning the metric structure of the spaces of walks and the identification with spaces of Lipschitz functions. Section \ref{s3} will be focus on the analysis of the dual space of the (metric) space of walks, and we will show its main properties, including compactness results. The Section \ref{s4} is devoted to introduce the concept of proximity, and to study some representations that can be obtained through standard measure theory results. Finally, 
in Section \ref{s5} 
we show some examples that highlight the most important concepts and results presented before, and discuss several applications. 

Let us give now some basic definitions and notations. 
Recall that a \emph{graph} $G$ is a pair $(V,E)$ where $V$ is a nonempty set and $E\subset V\times V$. The elements of $V$ and $E$ are called \emph{vertices} and \emph{edges} respectively. We require that $(v,v)\in E$ for all $v\in V$. A \emph{walk} $w$ is a sequence $w\colon\mathbb N\to V$ such that $(w(i),w(i+1))\in E$ for all $i\in\mathbb N$. Note that the image of a walk can be finite, since we have required $(v,v)\in E$ for all vertices. We denote the set of all walks of the graph $G$ as $W(G)$. Moreover, a graph $G$ is \emph{connected} if for any vertices $v_1,v_2\in V$ there exists a walk $w\in W(G)$ such that $v_1,v_2\in w(\mathbb N)$. In addition, a graph $G$ is \emph{non-directed} if for any $(u,v)\in E$ then $(v,u)\in E$. From now on, we will work with non-directed connected graphs. We assume that since we will understand the graph $G$ as a metric space endowed with the shortest path distance (which is given by the minimum number of edges connecting two vertices), and it is always well defined if the graph is connected. In this case, we can consider $E$ as the pairs of verteces $(u,v)$ such that $d(u,v)\leq 1$. On the other hand, any graph endowed with this metric is a complete metric space. Either otherwise is stated, we will assume that the graphs are bounded, that is $\sup_{v_1,v_2} \, d(v_1,v_2) < \infty.$
 A rigorous exposition of graph-theoretic foundations is provided in \cite{bondy}.

On the other hand, it is said that a function $f:M \to N$ between metric spaces $(M,d)$ and $(N,\rho)$ is Lipschitz if it satisfies the Lipschitz inequality for a certain constant $K>0,$ that is,
$$
\rho(f(v_1),f(v_2)) \le K \, d(v_1,v_2), \quad v_1,v_2 \in M. 
$$
The infimum of such constants is the so called Lipschitz constant of $f,$ and it is denoted by $\|f\|_{Lip}$. It can be defined explicitly as 
$$
\|f\|_{Lip}=\sup\left\{\frac{|\rho(f(x))-\rho(f(y))|}{d(x,y)}:x,y\in M, x\neq y\right\}.
$$

The set of all Lipschitz functions from $M$ to $N$ is denoted by $\operatorname{Lip}(M,N)$, that is, $\operatorname{Lip}(M,N)=\{f\colon M\to N : \|f\|_{Lip}<\infty\}$. If the Lipschitz function $f$ is only known in certain subspace $M_0\subset M$, there exist techniques to extend the function to the entire space $M$ while preserving the Lipschitz constant $K$ in $M_0$. Some well-known extensions are the McShane and Whitney formulas, which can be computed explicitly as
$$F^M(x):=\sup_{y\in M_0}\{f(y)-Kd(x,y)\}, \quad F^W(x):=\inf_{y\in M_0}\{f(y)+Kd(x,y)\},$$
respectively. Furthermore, for any other extension $F$ of $f$ yields $F^M\leq F\leq F^W$, and $F_\alpha:=(1-\alpha)F^M+\alpha F^W$ is another extension of $f$ for each $0\leq\alpha\leq 1$.

The space $\operatorname{Lip}(M,\mathbb R)$ is a vectorial subspace of the real continuous functions in $M$. However, it is not a normed space considering the norm provided by the Lipschitz constant. To this end, it is defined the space $$\operatorname{Lip}_0(M,\mathbb R)=\{f\in\operatorname{Lip}(M,\mathbb R):f(\theta)=0\},$$
where $\theta\in M$ is a fixed element in $M$. In this space, $\|\cdot\|_{Lip}$ defines a norm such that $\operatorname{Lip}_0(M, \mathbb R)$ equipped with it is a Banach space. An important observation is that very point $x\in M$ can be represented by the continuous linear operator $\delta_x\colon \operatorname{Lip}_0(M, \mathbb R)\to\mathbb R$ defined as $\delta_x(f):=f(x)$. The \emph{free space} of $M$, denoted by $\mathcal{F}(M)$, is defined as the closed linear span $\mathcal{F}(M):=\overline{\operatorname{span}}\{\delta_x:x\in M\}$ in $\operatorname{Lip}_0(M,\mathbb R)^*$. The norm in the free space, which is also known in the literature as the Arens-Eells space $AE(M)$, is the induced by duality: $$\|m\|_{\mathcal F}=\sup_{\|\varphi\|_{Lip}\leq 1} |m(\varphi)|.$$ Then, the map $\delta\colon M\to\mathcal F(M)$ defined as $\delta(x):=\delta_x$ is an isometric embedding from the metric space $M$ to the Banach space $\mathcal F(M)$. Furthermore, the free space is the predual of $\operatorname{Lip}_0(M,\mathbb R)$, that is, $\mathcal F(M)^*\cong \operatorname{Lip}_0(M,\mathbb R)$.

We refer to \cite{cob} for general questions on Lipschitz functions, and \cite{weaver2018lipschitz} for more information about the free space.

\section{The metric space of walks}\label{s2}

The aim of this section is to present how we understand the space of walks $W(G)$ in a graph $G$ as a metric space. In a general context, this aim is framed in the general and important problem of determining how similar two trajectories are, which has been widely discussed (see \cite{driemel2022discrete,su2020survey}). Some of the ideas that can be found relate, broadly speaking, to how to ``deform" one of the trajectories or walks into the other and then measure this change.

We will consider the following metric for walks in a bounded, non-directed and connected graph $G=(V,E)$.  For a fixed positive sequence $\tau=(\tau_i)_i$ such that $\sum_{i=1}^\infty \tau_i=1$, we define the distance between walks $w,u\in W(G)$ as
$$
d_\tau \big( w, u  \big)= \sum_{i=1}^\infty \tau_i \, d \big(w(i),u(i)\big).
$$

This metric aligns the vertex $i$ of the walk $w$ with the corresponding vertex of $u$, so that the distance between them is $d(w(i),u(i))$. But the alignment of two vertices can be more or less expensive, apart from their distance, depending on what they represent. Therefore, $\tau_i$ will represent the cost of this alignment.

Before showing that $d_\tau$ is indeed a metric and some properties of the resulting metric space, let us shown an important characterization of the space of walks. In particular, we will see that any walk can be seen as an element of the unit closed ball of Lipschitz sequences $\operatorname{Lip}(\mathbb N,V)$, that is, $B_{Lip}(G)=\{\alpha\colon\mathbb N\to V:\|\alpha\|_{Lip}\leq 1\}$. To show this, we will  understand $\mathbb N$ as a metric subspace of $\mathbb R$ with the usual euclidean distance.

\begin{lemma}\label{le1}
Consider a graph $G$ endowed with the shortest path distance $d.$ Then, the set of nonconstant walks can be identified with the unit sphere $S_{Lip}(G)=\{\alpha\colon\mathbb N\to V:\|\alpha\|_{Lip}=1\}$.

\begin{proof}
For any walk $w\in W(G)$ and $i,j\in\mathbb N$, the verteces $w(i)$ and $w(j)$ are connected by the self walk $w$ with length $|i-j|$. Thus, by definition of $d$ we get
\begin{equation}\label{eq:lipin}
    d(w(i),w(j))\le |i-j|,
\end{equation}
and therefore $\|w\|_{Lip}\le 1$. Suppose now that $\|w\|_{Lip}< 1$. If the walk is nonconstant, there exists $i\in\mathbb N$ such that $w(i)\neq w(i+1)$, and hence $$1=d(w(i),w(i+1))\leq \|w\|_{Lip} <1,$$
which is a contradiction. Consequently, for each nonconstant walk $w$ we have that $\|w\|_{Lip}=1$. For the converse, take a Lipschitz sequence $w\in S_{Lip}(G).$ By definition we get $d(w(i),w(i+1))\leq 1 |i-i-1|=1$ and thus $(w(i),w(i+1))\in E$. Therefore, $w\in W(G)$.
\end{proof}
\end{lemma}

\begin{proposition}\label{prop:walklip}
    The set of walks $W(G)$ can be identified as $B_{Lip}(G)$.
    \begin{proof}
        To get the result it is enough to consider that $B_{Lip}(G)$ contains only sequences with Lipschitz constant 1 or 0. Indeed, for $w\in B_{Lip}(G)$ such that $\|w\|_{Lip}<1$ then $d(w(i),w(i+1))\leq \|w\|_{Lip}<1$ for any $i\in\mathbb N$. Therefore, $d(w(i),w(i+1))=0$ and so $w$ is a constant walk since $w(i)=w(i+1)$ and hence $\|w\|_{Lip}=0$. So, the unitary closed ball $B_{Lip}(G)$ contains the constant walks and the nonconstant (in its sphere) by Lemma \ref{le1}. On the other hand, considering \eqref{eq:lipin} we also get that every walk is in $B_{Lip}(G)$.
    \end{proof}
\end{proposition}

\begin{proposition}\label{prop:com}
The function $d_\tau: W(G) \times W(G) \to \mathbb R$ is a well-defined metric, and $(W(G),d_\tau)$ is a complete metric space.
\begin{proof}
For any two walks $w,u \in W(G)$ holds
$$
d_\tau(w,u)= \sum_{i=1}^\infty \tau_i d(w(i),u(i)) \le \big( \sum_{i=1}^\infty \tau_i \big) \, \sup_{v_1,v_2 \in V} \, d(v_1,v_2)
<\infty,
$$
and so it is well-defined. In addition, it is straightforward that the function satisfies the triangle inequality and symmetry. Finally, suppose that $d_\tau(w,u)=0.$ Then, for all $i \in \mathbb N,$ we have that $d(w(i),u(i))=0$ since $\tau_i>0$ and so
$w(i)=u(i)$. Therefore, $w=u$.

On the other hand, consider a Cauchy sequence $(w_n)_n$ in $W(G).$ Then, for every $\varepsilon >0$ there exist an index $n_0$ such that for all $n,k \ge n_0$ we get $d_\tau(w_n,w_k) < \varepsilon.$
In particular, this means that for all $i \in \mathbb N,$ $d(w_n(i),w_k(i)) < \varepsilon/\tau_i.$ Consequently, since the parameter $\tau_i$ is fixed for a fixed $i,$ we conclude that, 
$$
\lim_{n,k} d(w_n(i),w_k(i)) = 0,
$$
and so $(w_n(i))_n$ is a Cauchy sequence in $G$. Moreover, since $(G,d)$ in complete, there exists the limit in $G,$ denote it $w_0(i)$ (that is, $w_0$ is defined pointwise as all the pointwise limits given above). Let us show now that $w_0 \in W(G).$ For every $i\in \mathbb N$ and $\varepsilon >0,$ 
there exists $n\in\mathbb N$ such that $d(w_n(i),w_n(i+1))<\varepsilon/2$. Since $\|w_n\|_{Lip} \le 1$ we get
\begin{align*}
     & d(w_0(i),w_0(i+1)) \\ \le & \,\, d(w_0(i),w_n(i)) + d(w_n(i),w_n(i+1))+ d(w_n(i+1),w_0(i+1)) \\
     \le& \,  \, \varepsilon + 1.
\end{align*}
From this we conclude that $d(w_0(i),w_0(i+1))\leq 1$ and so $(w_0(i),w_0(i+1))\in E$. That is, $w_0\in W(G)$ as desired.

\end{proof}
\end{proposition}

Although we have distinguished between constant and nonconstant walks in the previous results, we will also study the special case of eventually constant walks. A walk $w\in W(G)$ is eventually constant if there exist $i_0\in\mathbb N$ such that $w(i)=w(i+1)$ for all $i\geq i_0$. Note that every constant walk is also eventually constant for $i_0=1$. The following results will explore this type of walk.

\begin{lemma}
A walk $w\in W(G)$ is eventually constant if and only if there exist a vertex $v \in V$ such that $\lim_i d(w(i),v) = 0.$
\end{lemma}
The proof is  straightforward.
It is important to note that the limit of a convergent sequence of eventually constants graphs $(w_n)_n$ is not necessarily eventually constant. Consider the following example, based on a graph with two directly connected vertices $v_1$ and $v_2.$ Define
$$
w_1(1)=v_1, \, w_1(i)=v_2, \quad i \ge 2, 
$$
$$
w_2(1)=v_1, \, w_2(2)=v_2, \, w_2(i)=v_1, \quad i \ge 3,
$$  
$$
w_3(1)=v_1, \, w_3(2)=v_2, \, w_3(3)=v_1, \, w_3(i)= v_2 \quad i \ge 4,
$$ 
$$
\cdots \cdots
$$
In general, for every $n \in \mathbb N,$
$$
w_n(i) =
\begin{cases}
v_1 & \text{if } i \text{ is odd}, \\
v_2 & \text{if } i \text{ is even}
\end{cases}
\quad \text{for } i \leq n, 
$$
\text{and}
$$w_n(i) =
\begin{cases}
v_1 & \text{if } n \text{ is even}, \\
v_2 & \text{if } n \text{ is odd}
\end{cases}
\quad \text{for} \ i > n.
$$
The limit of this sequence of eventually constant walks is the walk $w_0$ given by $w_0(i)=v_1$ when $i$ is odd and $w_0(i)=v_2$ when $i$ is even. Indeed, note that for every $n\in\mathbb N$ we have that $d(w_n(i),w_0(i))=0$ if $i\leq n$, and 
$d(w_n(i),w_0(i))\leq 1$ for $i>n$. Therefore, $$d_{\tau}(w_n,w_0)=\sum_{i=1}^\infty \tau_id(w_n(i),w_0(i))=\sum_{i=n+1}^\infty \tau_id(w_n(i),w_0(i))\leq \sum_{i=n+1}^\infty \tau_i,$$
and since the series of $(\tau_i)_i$ is convergent we conclude that $\lim_n d_\tau(w_n,w_0)=0$.

\section{Duality in the space of walks} \label{s3}
The aim of this section is to define a space of real functions which can be identified with the ``dual space'' of the metric space of walks defined in the previous section. After developing and proving some basic results of this notion of dual functions, which we call evaluations of the graph, we will describe how the duality between the spaces should work.

 For a graph $G=(V,E),$ we define the space of evaluations of the graph as the space $E(G):= \operatorname{Lip}_0(V,\mathbb R)$, that is, the real Lipschitz functions $\varphi:V \to \mathbb R$ such that $\varphi(v_0)=0$, where $v_0 \in V$ is a fixed chosen vertex of the graph. Recall that this space endowed with the Lipschitz norm is a Banach space. An important observation is that the evaluations of the graph are bounded functions, since for every $v \in V$ we have that
$$
\sup_{v\in V}|\varphi(v)|=\sup_{v\in V}| \varphi(v)- \varphi(v_0)| 
$$
$$
\le \|\varphi\|_{Lip} \, \sup_{v\in V}d(v,v_0) 
\le \|\varphi\|_{Lip} \, \sup_{v_1,v_2 \in V} \, d(v_1,v_2) < \infty.
$$

Let us show first how these functions interact with the space of walks $W(G)$. To do that, we will consider for $\varphi \in E(G)$ and $w \in W( G)$ the composition $\varphi \circ w: \mathbb N \to \mathbb R.$ We write as $\varphi \circ W(G)$ the space all these sequences.
\begin{lemma} \label{1dual}
Consider a metric graph $(G,d)$ and the fixed sequence $\tau=(\tau_i)_i.$

\begin{itemize}
\item[(i)] The sequences $\varphi \circ W(G)$ are bounded and $\|\varphi \circ w\|_{Lip} \le \|\varphi\|_{Lip}$ for every $w \in W(G).$

\item[(ii)] For every $\varphi \in E(G)$ yields
$
\|\varphi\|_{Lip}= \sup_{w \in W(G)} \|\varphi \circ w\|_{Lip}.
$

\item[(iii)] Let $\mathcal{B}(\mathbb R)$ be the space of bounded sequences in $\mathbb R$ equiped with the metric $d_{\mathbb R,\tau}((a_i),(b_i))=\sum_i\tau_i|a_i-b_i|$. Then, the inclusion map  $\Upsilon: \big( \varphi \circ W(G), d_\tau \big) \to \big( \mathcal B(\mathbb R), d_{\mathbb R,\tau} \big)$ is  Lipschitz and $\|\Upsilon\|_{Lip} \le \|\varphi\|_{Lip}.$


\end{itemize}

\end{lemma}
\begin{proof}
Since the functions $\varphi\in E(G)$ are bounded, then $\varphi\circ w$ is also bounded for every $w\in W(G)$. Furthermore, since $\|f\circ g\|_{Lip}\leq \|f\|_{Lip}\|g\|_{Lip}$ for any composition of Lipschitz functions $f$ and $g$, and since $\|w\|_{Lip}\leq 1$ by Proposition \ref{prop:walklip}, we get the result of (i).

On the other hand, fix $\varphi \in E(G).$ For $\varepsilon >0$ there exist two different vertices $v_1$ and $v_2$ such that 
$$
\|\varphi\|_{Lip} -\varepsilon <  \frac{\big| \varphi(v_1)- \varphi(v_2) \big|}{d(v_1,v_2)}.
$$
Since the graph $G$ is connected, there exists a walk $w\in W(G)$ such that $w(i)=v_1$, $w(j)=v_2$ and $d(v_1,v_2)=|i-j|$.
Then,
\begin{align*}
    \|\varphi\|_{Lip} -\varepsilon &<
  \frac{\big| \varphi(v_1)- \varphi(v_2) \big|}{d(v_1,v_2)}
  =
  \frac{\big| \varphi(w(i))- \varphi(w(j)) \big|}{|i-j|}\\
  &\le \sup_{w \in W(G), \, i \ne j}  \frac{\big| \varphi(w(i))- \varphi(w(j)) \big|}{|i-j|} \\ &= \sup_{w \in W(G)} \, \|\varphi \circ w\|_{Lip}
\end{align*}
Since this holds for every $\varepsilon>0$ we get $\|\varphi\|_{Lip}\leq sup_{w\in W(G)} \|\varphi\circ w\|_{Lip}$, and by (i) we conclude (ii).

Finally, for $w, u\in W(G)$ note that, since $\varphi$ is Lipschitz, we get
$$
\sum_{i=1}^\infty \tau_i \big|\varphi(w(i))- \varphi(u(i)) \big|
\le \|\varphi\|_{Lip} \, \sum_{i=1}^\infty \tau_i \, d(w(i),u(i)),
$$
which gives (iii).
    
\end{proof}

Lemma \ref{1dual} opens the door to define a duality relation between the space of walks in the graph $W(G)$ and the space of Lipschitz evaluations of the graph $E(G)$. Using the traditional notation for duality, we write
$$
\langle w, \varphi \rangle = \sum_{i=1}^\infty \tau_i \, \varphi(w(i)),
\qquad w \in W(G), \,\, \varphi \in E(G).
$$
Then, we define 
$$
\langle w_1 \circleddash w_2, \varphi \rangle := \langle w_1, \varphi \rangle - \langle w_2, \varphi \rangle.
$$
We can extend this definition to some sort of restrictions of walks on subsets of $\mathbb N$ as follows. If $\Sigma$ is the $\sigma-$algebra of all the subsets of natural numbers, we can define for $w \in W(G)$ and $A \in\Sigma$ the set function $A \mapsto w(A)$ in $\operatorname{Lip}(\mathbb N, V)$ by
$$
w(A)(i) := w(i) \quad \text{if} \quad i \in A, \quad \text{and} \quad
w(A)(i)= v_0 \quad \text{if} \quad i \notin A.
$$
Thus,
$$
\langle w, \varphi \rangle (A) = \sum_{i \in A} \tau_i \, \varphi(w(i)),
\qquad w \in W(G), \,\, \varphi \in E(G),
$$
and
$$
\langle w_1 \circleddash w_2, \varphi \rangle(A) = 
\langle w_1(A) \circleddash w_2(A), \varphi \rangle
=\langle w_1(A), \varphi \rangle - \langle w_2(A), \varphi \rangle.
$$

\begin{lemma}
If $\varphi \in E(G),$ the following equalities hold.
\begin{itemize}
    \item[(i)] The Lipschitz norm of $\varphi$ can be computed by
$$\|\varphi\|_{Lip} = \sup \left\{\frac{\big| \langle w_1 \circleddash w_2(A), \varphi \rangle \big|}{d_\tau(w_1,w_2)}: w_1,w_2 \in W(G), \, w_1 \ne w_2, \, A \in \Sigma \right\}.
$$
\item[(ii)] $\sup_{\|\varphi\|_{Lip}\leq 1}\big| \langle w_1 \circleddash w_2, \varphi \rangle \big| \le {d_\tau(w_1,w_2)}$
\end{itemize}
\end{lemma}
\begin{proof}
According to the definition, we obtain
\begin{align*} 
    \big| \langle w_1 \circleddash w_2(A), \varphi \rangle \big| &= \big|\langle w_1(A), \varphi \rangle - \langle w_2(A), \varphi \rangle \big| \\
& \le
\Big| \sum_{i \in A} \tau_i \, \varphi(w_1(i)) - \sum_{i\in A} \tau_i \, \varphi(w_2(i)) \Big| \\
&\le  \sum_{i\in A} \tau_i \, \big| \varphi(w_1(i)) -   \varphi(w_2(i)) \big|\\ &\le \|\varphi\|_{Lip} \, \sum_{i\in A}^\infty \tau_i d(w_1(i),w_2(i))\\ & \le \|\varphi\|_{Lip} \, d_\tau(w_1,w_2). 
\end{align*}
Now, for $\varepsilon >0$ take a pair $v_1,v_2 \in V$ such that
$$
\|\varphi\|_{Lip} - \varepsilon < \frac{ | \varphi(v_1)- \varphi(v_2) |}{d(v_1,v_2)}.
$$
Then, it is enough to consider the constant walks $w_1(i)= v_1$ and
$w_2(i)=v_2$ for all $i \in \mathbb N$ to get
\begin{align*}
    \|\varphi\|_{Lip} - \varepsilon &< \frac{ | \varphi(v_1)- \varphi(v_2) |}{d(v_1,v_2)} \\
    & = \frac{ | \sum_{i=1}^\infty \tau_i \,  \varphi(w_1(i))-  \sum_{i=1}^\infty \tau_i \varphi(w_2(i)) |}{ \sum_{i=1}^\infty \tau_i d(w_1(i),w_2(i))}\\
&= \frac{\big| \langle w_1 \circleddash w_2, \varphi \rangle \big|}{d_\tau(w_1,w_2)}.
\end{align*}
From this, note that we only need to consider $A= \mathbb N$ to conclude the result of (i):
$$
\sup_{w_1,w_2 \in W(G), w_1 \ne w_2, A} \frac{\big| \langle w_1 \circleddash w_2 (A), \varphi \rangle \big|}{d_\tau(w_1,w_2)}
=\|\varphi\|_{Lip}.$$
Furthermore, taking the supremum in the inequalities at the beginning of the proof when $\|\varphi\|_{Lip}\le 1$, we conclude (ii):
$$\sup_{\|\varphi\|_{Lip}\leq 1} |\langle w_1 \circleddash w_2,\varphi\rangle|\leq\sup_{\|\varphi\|_{Lip}\leq 1} \|\varphi\|_{Lip}d_\tau(w_1,w_2)=d_\tau(w_1,w_2).$$
\end{proof}
 






\begin{lemma} \label{precom}
Fix $v_0 \in V$ and consider the constant walk given by $w_0(i)=v_0$ for all $i \in \mathbb N$. For every $A \in \Sigma,$ the map
$$
\iota_A \colon W(G) \to \mathcal F(G), \quad w \mapsto \iota_A(w)=\sum_{i \in A} \tau_i\delta_{w(i)}
$$
where $\mathcal{F}(G)$ is the free space and $w = (w(i))_i,$ is well defined, continuous, and Lipschitz with $\|\iota_A\|_{Lip} \leq 1$.
\end{lemma}
\begin{proof}
 Note that, for all $n \in \mathbb{N}$,
$$
s_n(w):=\sum_{i \in A \cap \{1,\dots ,n\}} \tau_i \delta_{w(i)} \in \mathcal F(G)
$$
and
$$
\left\|s_n(w)\right\|_{\mathcal F} \le\sum_{i\in A \cap \{1,\dots ,n\}}\tau_i\,d\bigl(w(i),v_0\bigr) < \infty,
$$
since the graph is bounded.
Moreover, for $m>n,$
$$
\|s_m(w)-s_n(w)\|_{\mathcal F}
      \le\sum_{i\in A\cap\{n+1,\dots ,m\}}\tau_i\,d\bigl(w(i),v_0\bigr)
      \xrightarrow[n\to\infty]{}0,
$$
since the series $\sum_{i\in A}\tau_i d\bigl(w(i),v_0\bigr)$ converges.

Therefore, the sequence of partial sums $(s_n(w))_{n\in\mathbb N}$ is Cauchy and, since $\mathcal{F}(G)$ is complete, its limit
$$
\iota_A(w) = \sum_{i \in A} \tau_i \delta_{w(i)}
$$
exists and belongs to $\mathcal{F}(G).$ So $\iota_A$ is well defined.

Furthermore, if $w_1, w_2 \in S_1(G)$, we have
$$
\|\iota_A(w_1) - \iota_A(w_2)\|_{\mathcal F} \leq \sum_{i \in A} \tau_i d(w_1(i), w_2(i)) = d_\tau(w_1(A), w_2(A)),
$$
so the map $\iota_A$ is Lipschitz with $\mathrm{Lip}(\iota_A) \leq 1$.
\end{proof}

\begin{remark}
Once a (piece of) walk $w(A)$ is identified with an element of $\mathcal F(G),$ we can compute its norm by duality. That is, 
$$
\|\iota_A(w)\|_{\mathcal F} = \sup_{\|\varphi\|_{Lip}\le 1 } | \langle \iota_A(w), \varphi \rangle |.
$$
\end{remark}

\vspace{0.3cm}



\begin{theorem} \label{compact}
Consider the topological space $(E(G), T_p)$ endowed with the topology $T_p$ of the pointwise convergence of the functions
$$
\varphi \mapsto \langle w_1 \circleddash w_2(A), \varphi \rangle,
$$
where $w_1, w_2 \in W(G)$ and $A \in \Sigma.$ Then,
\begin{itemize}
\item[(i)] the space
 $(B_{Lip_0(G)}, T_p)$ is compact, and
\item[(ii)] the topologies
 $T_p$ and  weak* on $B_{Lip_0(G)}$ coincide.
\end{itemize}
\end{theorem}
\begin{proof}
As shown in Lemma \ref{precom}, the elements $w\in W(G)$ as well as the ``differences"
$w_1 \circleddash w_2(A)$ (couples of walks restricted to $A \subset \mathbb N$,) can be identified with molecules of the free space $\mathcal F(G).$ Indeed, note that the evaluation $$
\langle w_1 \circleddash w_2(A), \varphi \rangle 
$$
can be identified with the action $\langle \iota_A(w_1)- \iota_A(w_2),\varphi \rangle$, where the duality in the last formula is understood as the duality between molecules $\iota_A(w_1)$ and $\iota_A(w_2)$ and the elements of $\operatorname{Lip}_0(V,\mathbb R).$

Note also that the functions $w_1$ separates the elements of $\operatorname{Lip}_0(V,\mathbb R)$ since, if $\varphi_1 \ne \varphi_2,$ there exist a vertex $v\in V$ such that
$\varphi_1(v) \ne \varphi_2(v),$ and $v$ can be identified with the constant walk $w_1(i)=v.$ 

Therefore, the topology $T_p$ is Hausdorff and
weaker than the weak* topology for the duality $\mathcal F(G)^*=\operatorname{Lip}_0(V,\mathbb R)$, providing that $(B_{\operatorname{Lip}}(G), T_p)$ is compact. Therefore, the identity map $\iota\colon(B_{\operatorname{Lip}}(G), w^*)\to(B_{\operatorname{Lip}}(G), T_p)$ is bijective and continous from a compact to a Hausdorff space, so we can conclude that its inverse is continous. We conclude therefore that both topologies muest coincide.
\end{proof}

\section{Proximities in spaces of walks with integral representations and approximation} \label{s4}
In this section, we aim to study the similarity between walks in a graph using various analytical tools based on the results obtained in the previous section. We will do it through the concept of proximity, that is a function acting on pairs of walks that is small when the distance between the elements in the pair  is small too.  Inspired by the standard concepts used in computer science and machine learning to measure similarities between analyzed objects, we address the problem of finding a general procedure to represent proximities through a standard method and under a unified perspective. Although the typical case we have in mind involves the natural metrics $d_\tau$ on the space, the instruments commonly used to measure distances in these fields are often not directly related to metrics (in fact, they are generally not metrics, as in the case of the highly popular cosine similarity). This is the reason we prefer to refer to these functions as proximities. The aim of this section is to obtain general representations of such proximities which satisfy certain properties. We will use some standard notions of measure theory, that the reader can find in any advanced book on the topic as
\cite{halmos}.




Due to the convenience of comparing pieces of walks as well as complete walks, we define proximities as functions on the class of all the restrictions of the walks to subsets of indices.
Thus,  since the walks are taken to be a particular class of pointwise evaluations of functions on the graph, we can also understand them as set functions. In order to do this,
consider the $\sigma-$algebra $\Sigma$ defined by all the subsets of $\mathbb N.$ Recall that we are considering a fixed vertex $v_0 \in V$ and, if $w \in W(G)$ and $A \in \Sigma,$ we consider the set function $A \mapsto w(A)$ in $\operatorname{Lip}(\mathbb N, V)$ by
$$
w(A)(i) := w(i) \quad \text{if} \quad i \in A, \quad \text{and} \quad
w(A)(i)= v_0 \quad \text{if} \quad i \notin A.
$$

Recall that we have defined the distance $d_\tau$ for a positive sequence $(\tau_i)_i$ such that $\sum_{i=1}^\infty\tau_i=1$ by
$d_\tau(w,u)= \sum_{i=1}^\infty \tau_i \, d \big( w(i),u(i) \big);$ clearly, we can define it essentially by the same formula for restricted walks $w(A)$ and $u(A),$ $A \in \Sigma,$ as
$$
d_\tau(w(A),u(A))= \sum_{i=1}^\infty \tau_i \, d \big( w(A)(i),u(A)(i)) \big)
=\sum_{i \in A}^\infty \tau_i \, d \big( w(i),u(i)) \big).
$$ 

\begin{definition}
Let $\Sigma_0$ be a sub-$\sigma$-algebra of $\Sigma.$
We define a proximity $P: W(G) \times W(G) \times \Sigma_0 \to \mathbb R$ as a family of functions $P(\cdot,\cdot,A): W(G) \times W(G) \to \mathbb R$  indexed by the class of sets $A \in \Sigma_0$ such that, for a fixed $A \in \Sigma_0,$
\begin{itemize}
    \item[(i)]
$P(w_1,w_2,A)$ is a function of $w_1(A)$ and $w_2(A)$ for each $w_1,w_2 \in W(G)$, and
\item[(ii)] $P$ is continuous with respect to $d_\tau.$
    
\end{itemize}
That is, for a sequence of walks $(w_n)_n,$ a fixed walk $u$ and a set $A \in \Sigma_0$ such that $\lim_n d_\tau(w_n(A),u(A))=0,$
then $\lim_n P(w_n(A),u(A))=0$ too. 
\end{definition}

We will say that a proximity is additive if for every pair $w_1,w_2\in W(G),$ then $P(w_1,w_2, A \cup B)= P(w_1,w_2,A)+ P(w_1,w_2,B)$ for disjoint $A,B \in \Sigma_0.$

In what follows, we develop the necessary tools to obtain a canonical representation of such functions. The idea is to show that, under certain basic requirements, each proximity can be written as a weighted variation of the fundamental proximity instruments, that are on the one hand the ones provided by the distance $d_\tau(w_1(A),w_2(A)),$ and, on the other hand, the evaluations of functions as $P_{\varphi_0}(w^1,w^2,A)= \big| \langle w_1 \circleddash w_2, \varphi_0 \rangle (A) \big|$ for a fixed $\varphi_0\in E(G).$ Let us start with the distance-type proximities.

\subsection{Proximities dominated by the distance $d_\tau$} \label{s4.1}

Consider a sub-$\sigma-$algebra $\Sigma_0$ of $\Sigma, $ and let
 $P: W(G) \times W(G) \times \Sigma_0 \to \mathbb R$ be a proximity bounded by $d_\tau,$ that is, it satisfies that there is a constant $K>0$ such that for every sequence of couples of walks $w_1,w_2\in W(G)$ and sets $A\in\Sigma_0$,
$$
P(w_1,w_2,A) \le K \, d_\tau(w_1(A),w_2(A))= K \, \sum_{i \in A} \tau_i d(w_1(i),w_2(i)).
$$

\begin{proposition}
Let $P$ be an additive proximity that is bounded by $d_\tau$ with constant $K>0.$
For every $w_1,w_2\in W(G),$ there exist a sequence $(s^{w_1,w_2}_i)_i {\subset \mathbb R}$  bounded by $K$ such that
$$
P(w_1,w_2,A)= \sum_{i \in A} s^{w_1,w_2}_i \, \tau_i \, d(w_1(i),w_2(i)).
$$
\end{proposition}
\begin{proof}
For fixed $w_1,w_2,$ the function $A \mapsto P(w_1,w_2,A)$ is additive. Since it is bounded with respect to $d_\tau,$ we have that this set function is a countably additive measure. The Radon-Nikodym Theorem gives a $K-$bounded integrable function (a sequence) $(s^{w_1,w_2}_i)_i$ such that
$$
P(w_1,w_2,A)= \sum_{i \in A} s^{w_1,w_2}_i \tau_i \, d(w_1(i),w_2(i)),
$$
so we get the result.
\end{proof}


\subsection{Proximities defined by duality} \label{s4.2}
According to the duality we are considering among the metric space $W(G)$ and the Banach space of functions $E(G)$, we can consider the evaluation functions $\varphi \mapsto \langle w_1(A) \circleddash w_2(A), \varphi \rangle$ defined for continuous functions in the unit ball of $E(G)$, which we write $C(B_{E(G)})$ or simply $C(B).$ Thus, the requirement in the next lemma is a sort of $1-$concavity inequality for a fixed proximity function. Note that, in particular, we can consider $ \Sigma_0 = \{ \mathbb{N}, \emptyset \} $, and in this case, what we obtain is a proximity function defined only for complete walks, not for parts of them.

\begin{lemma} \label{1conca}
Consider a sub-$\sigma-$algebra $\Sigma_0$ of $\Sigma.$
Suppose that there is a proximity $P: W(G) \times W(G) \times \Sigma_0 \to \mathbb R$ that satisfies that  there is a constant $K>0$ such that for every sequence of couples of walks and sets $(w^1_1,w^2_1,A_1), ..., (w^1_n, w^2_n,A_n),$
$$
\sum_{k=1}^n P(w^1_k,w^2_k,A_k) \le K \ \Big\|
 \sum_{k=1}^n \big| \langle w^1_k \circleddash w^2_k, \cdot \rangle (A_k) \big| \Big\|_{C(B)}.
$$
Then there is a function $\varphi_0 \in B_{E(G)}$ such that for every couple of walks $w_1,$ $w_2$ and $A \in \Sigma_0,$
$$
P(w_1,w_2,A) \le K \big| \langle w_1 \circleddash w_2, \varphi_0 \rangle (A) \big| 
$$
\end{lemma}
\begin{proof}
First, let us recall that there is a well-known trick for extending the inequality in the statement of the theorem to the same inequality for convex combinations. It is based on the fact that we can repeat terms in the sum as many times as needed, and divide by the total number of terms, in such a way that we obtain a rational approximation for any convex combination by means of rational numbers, and thus a limit procedure yields the result. This procedure can be found well explained in \cite{far}.
Thus, from the original inequalities  we obtain that, for every set $\alpha_1,...,\alpha_n >0$ such that $\sum_{k=1}^n \alpha_k=1,$
\begin{align*}
\sum_{k=1}^n \alpha_k P(w^1_k,w^2_k,A_k) &\le K \ \Big\| \sum_{k=1}^n \alpha_k \big| \langle w^1_k \circleddash w^2_k, \cdot \rangle (A_k) \big| \Big\|_{C(B)}\\
 = & \, K \ \sup_{\varphi \in B_{E(G)}}
\Big( \sum_{k=1}^n \alpha_k \sum_{i=1}^\infty \tau_i| \varphi(w^1_k(A_k)(i))- \varphi(w^2_k(A_k)(i))| \Big).
\end{align*}
for every finite sequence of paths, and we want to get a function $\varphi_0 \in B_{E(G)}$ such that 
$$
P(w^1,w^2,A) \le K \sum_{i=1}^\infty \tau_i| \varphi_0(w^1(A)(i))- \varphi_0(w^2(A)(i))|.
$$
For all sets of walks $w^1_k,$ $w^2_k,$ sets $A_k$ and parameters $\alpha_k$'s as above, we consider the functions
$$
\psi(\varphi)=\sum_{k=1}^n \alpha_k P(w^1_k,w^2_k,A_k) - K \ 
 \sum_{k=1}^n \alpha_k \sum_{i=1}^\infty \tau_i| \varphi(w^1_k(A_k)(i))- \varphi(w^2_k(A_k)(i))|,
$$
$\varphi \in B_{E(G)}.$
We can apply a direct Hahn-Banach separation argument to obtain the result. 
Each such function \( \psi \) is affine and continuous on $ B_{E(G)},$
that is compact by Theorem \ref{compact}. The set of all such \( \psi \) defines a convex subset of the dual space \( C(B_{E(G)})^* \), which can be separated from the origin using the Hahn-Banach separation theorem.

Therefore, there exists a function $\varphi_0 \in B_{E(G)}$ such that
$$
\psi(\varphi_0) \le 0 \quad \text{for all these } \psi,
$$
which implies
$$
P(w_1, w_2,(A_k)) 
$$
$$\le K \sum_{i=1}^\infty \tau_i \left| \varphi_0(w_1(A_k)(i)) - \varphi_0(w_2(i)) \right|
= K \left| \left\langle w_1(A_k) \circleddash w_2(A_k), \varphi_0 \right\rangle \right|,
$$
for all couples $w_1, w_2 \in W(G)$ and every set $A \in \Sigma_0$ as desired.
\end{proof}

After Lemma \ref{1conca}, we will say that 
a proximity $P$ is $1-$concave if it  satisfies  that there is a constant $K>0$ such that for every sequence of couples of walks $(w^1_1,w^2_1), ..., (w^1_n, w^2_n)$ and sets $A_1,...,A_n \in \Sigma,$
$$
\sum_{k=1}^n P(w^1_k,w^2_k, A_k) \le K \ \Big\|
 \sum_{k=1}^n \big| \langle w^1_k(A_k) \circleddash w^2_k(A_k), \cdot \rangle \big| \Big\|_{C(B)}.
$$

In what follows, and in order to find a general representation theorem for additive proximities, we show how to define measures by couples of walks.





Given two walks $w_1,w_2 \in W(G)$ and a function $\varphi \in E(G)$ we will consider set functions as
$$
A \mapsto \big|\langle w_1(A) \circleddash w_2(A), \varphi \rangle \big|
=
\sum_{i \in A} \tau_i|\varphi(w_1(i))- \varphi(w_2(i))| 
, \quad A \in \Sigma.
$$
The aim of what follows is to show that, under certain requirements, every proximity can be essentially written using a formula such as the one above.

\begin{lemma} \label{stdmeasure}
For $w_1, w_2 \in W(G)$ and $\varphi \in E(G)$, the map
$$
P_\varphi(w_1,w_2, A) = |\langle w_1(A) \circleddash w_2(A), \varphi \rangle|, \quad A \in \Sigma,
$$
defines a proximity which
is also a countably additive measure. 
\end{lemma}
The proof is straightforward.

Let $P$ be a proximity, and fix $w_1, w_2 \in W(G).$ Define the set function
$$
\widehat P(w_1,w_2)(A)= P(w_1,w_2,A), \quad A \in \Sigma.
$$
We say that it is additive if $P$ is additive. Moreover, if $\lim_n P(w_1,w_2)(\{k: k \ge n\})=0,$ the so-defined measure $\widehat P(w_1,w_2)(\cdot)$ is countably additive.
Next result shows that, under certain requirements, every additive proximity is essentially as the ones provided by Lemma \ref{stdmeasure}.

\begin{theorem} \label{main1con}
Let $P$ be a $1-$concave additive proximity with constant $K.$ Then  there is a function $\varphi_0 \in B_{E(G)}$ such that
 for every pair of walks $w_1,w_2$ there is a non-negative sequence $(s^{w_1,w_2}_i)_i {\subset\mathbb R}$ bounded by $K$ satisfying 
$$
P(w^1,w^2, A) = \sum_{i \in A} \tau_i \, s^{w_1,w_2}_i \, \big| \varphi_0(w_1(i))- \varphi_0(w_2(i)) \big|. 
$$
\end{theorem}
\begin{proof}
We consider the case $\Sigma= \Sigma_0.$ By hypothesis, there is a constant $K>0$ such that for every walks $(w^1_1,w^2_1), ..., (w^1_n, w^2_n)$ and $A_1,...,A_n \in \Sigma,$
$$
\sum_{k=1}^n P(w^1_k,w^2_k, A_k) \le K \ \Big\|
 \sum_{k=1}^n \big| \langle w^1_k(A_k) \circleddash w^2_k(A_k), \cdot \rangle \big| \Big\|_{C(B)}.
$$
Applying Lemma \ref{1conca}, we find a function $\varphi_0 \in C(B)$ such that 
$$
P(w^1,w^2, A) \le K \, \big| \langle w^1(A) \circleddash w^2(A), \varphi_0 \rangle \big|.
$$
By Lemma \ref{stdmeasure}, we have that the right hand side of this inequality gives a countably additive measure. Together with the inequality itself, this gives that $A \mapsto P(w^1,w^2, A),$ which is already additive, is also countably additive. The inequality, that works for all subsets $A,$ provides also that $P(w^1,w^2, A)$ is countably additive with respect to the (countably additive) measure
$$
A \mapsto \nu_{w_1,w_2}(A)=\big| \langle w^1(A) \circleddash w^2(A), \varphi_0 \rangle \big|, \quad A \in \Sigma.
$$
The Radon-Nikodym Theorem (or a direct calculation) gives a $\nu_{w_1,w_2}-$integrable function (a sequence) $s_{w_1,w_2}: \mathbb N \to \mathbb R$ such that
$$
P(w^1,w^2, A)= \int_A s^{w_1,w_2}_i \, d\nu_{w_1,w_2}(i)
$$
and $\sup_i s^{w_1,w_2}_i \le K,$
what gives the result. Taking into account that the measurable space is the set of natural numbers, we obtain that
$$
P(w^1,w^2, A)
$$
$$
= \sum_{i \in A} \tau_i \, s^{w_1,w_2}_i \, | \varphi_0(w_1(i))- \varphi_0(w_2(i)) | \le K \,
\sum_{i \in A} \tau_i \,  | \varphi_0(w_1(i))- \varphi_0(w_2(i)) | .
$$

\end{proof}

\begin{remark}
In particular, every proximity as the one given by Theorem \ref{main1con} is Lipschitz with constant less or equal to $K.$ But this inequality can be also written for the restriction of the walks to subsets $A,$ that is,
$$
P(w^1,w^2, A) = \sum_{i \in A} \tau_i \, s^{w_1,w_2}_i \, \big| \varphi_0(w_1(i))- \varphi_0(w_2(i)) \big|
\le K
\sum_{i \in A} \tau_i \, d(w_1(i),w_2(i))
$$
for all walks $w_1,w_2$ and $A \in \Sigma.$
\end{remark}

The formula for the proximities provided in Theorem \ref{main1con} are in a sense the canonical expression for the class of functions we are interested in what follows.  The result opens the door for applications of the results obtained to clustering.

\section{Applications: exploratory walks and proximity determination}\label{s5}

In this section we show some applications and examples related to the theoretical construction presented in the previous section. For the use of our results, the first problem is how to design proximities to solve concrete problems. The structural information provided by graph indices as well as walks-related indices can be used to define or refine the definition of proximity functions between walks, specifically tailored to solve certain problems (e.g. \cite{est}). For example, structural insights based on walks obtained through penalization schemes applied to walks (see \cite{del,ests,wei}) offer a more detailed understanding of the relationships between nodes and the topology of the graph. This type of information can be used to design more appropriate proximities in the space of walks.

On the other hand, since the spaces of walks contain many elements, it is sometimes difficult to find a priori a complete definition of proximities involving the whole structure. It is easier to find a good definition just for a metric subset of vertices or walks. In this section, we show how the use of classical extension techniques can facilitate the creation of approximation structures tailored for specific problems. In particular, since graphs can be considered as metric spaces, we can use the McShane-Whitney extension theorem for this purpose. It is worth noting that the study of extension for Lipschitz functions on graphs is a current topic of interest (see, e.g., \cite{Chand}).

Consider a graph $G$ containing two distinct vertices $v_1$ and $v_2$. Starting from a (finite) set of walks connecting $v_1$ and $v_2$, we are interested in design a proximity $P$ based on them. The objective is to find a classifier for the set of all walks in the graph connecting $v_1$ and $v_2$ considering that the affinity of two such walks is defined by $P$. Since we are connecting two given vertices, the term path would be more appropriate than the term walk; however, we retain the word walk to maintain consistency with the conceptual framework used in the other sections.

Suppose that there is a fixed additive $1-$concave proximity function $P$ with constant $K$ acting on $G.$ 
We know, as a consequence of the results of Section \ref{s4}, that there is a representation of $P$ as
$$
P(w^1,w^2, A) = \sum_{i \in A} \tau_i \, s^{w_1,w_2}_i \, \big| \varphi_0(w_1(i))- \varphi_0(w_2(i)) \big|
$$
for a certain sequence $(s^{w_1,w_2}_i)_i$ bounded by $K$ and a function $\varphi_0 \in B_{E(G)}.$ 

We will assume throughout this section that $P$ is known for (all the pieces of) a certain set of walks $W_0$, as well as all the elements that define $P$ at all the vertices involved in the walks of the set $W_0$, including $\varphi_0$.

\begin{remark} \label{remal}
Let us note that the general formula provided above for the proximity function $P$ gives a lot of information about its properties. First, note that we can compute all the weights $s_i^{w_1,w_2}$ associated to all the indices  $i$ for $w_1,w_2 \in W_0,$ by
$$
s_i^{w_1,w_2} = \frac{ P(w_1, w_2, \{i\}) }{\tau_i \left| \varphi_0(w_1(i)) - \varphi_0(w_2(i)) \right|},
$$
if $\varphi_0(w_1(i))\neq\varphi_0(w_2(i)).$ 
However, the function $\varphi_0$ is not known for the other vertices not involved in the range of the walks in $W_0.$ Recall that, in general, they are obtained via the Hahn–Banach theorem and so its explicit form is unknown. 
As a consequence, we cannot compute the other values $s_i^{w_1,w_2}$ directly.

Let us show, however, that we can get a explicit domination formula for $P.$
Nevertheless, by using the Lipschitz condition of $\varphi_0$ with respect to the weighted metric $d_\tau$, we can deduce a lower bound as follows:
$$
s_i^{w_1,w_2} \geq \frac{ P(w_1, w_2, \{i\})}{\|\varphi_0\|_{\mathrm{Lip}} \, \tau_i d(w_1(i), w_2(i))}.
$$
Rearranging this expression, we obtain an upper bound on the proximity term
$$
P(w_1, w_2, \{i\}) \leq \|\varphi_0\|_{\mathrm{Lip}} \cdot s_i^{w_1,w_2} \cdot \tau_i d(w_1(i), w_2(i)).
$$
Moreover, defining $s_i^{w_1,w_2}=0$ when $\varphi_0(w_1(i))=\varphi_0(w_2(i))$, note that this expression also holds since $P(w_1, w_2, \{i\})=0$. Therefore, summing over a measurable set $A$, we get
$$
P(w_1, w_2, A) \leq \|\varphi_0\|_{\mathrm{Lip}} \sum_{i \in A} s_i^{w_1,w_2} \cdot \tau_i d(w_1(i), w_2(i)).
$$

If we assume as  in previous sections that $\sup_i s_i^{w_1,w_2} \leq K$ for some constant $K$, then the proximity is controlled by the weighted distance
$$
P(w_1, w_2, A) \leq \|\varphi_0\|_{\mathrm{Lip}} \cdot K \sum_{i \in A} \tau_i d(w_1(i), w_2(i)).
$$

This shows that the proximities defined by duality in Section \ref{s4.2} are also dominated by the weighted metric distance $d_\tau$, as discussed in Section \ref{s4.1}.
\end{remark}



Let us show how can we define a proximity function starting from the formula provided in Theorem \ref{main1con} for $P$ and the experiments we have developed on the system, that give the set $W_0.$ Since we assume that we know the value of $P$ for all the (pieces of) the walks in $W_0,$ (and only for these walks) we can use its definition to construct a new proximity function for all the walks in the graph. In order to do that, we follow the next steps.

\begin{enumerate}
    \item We define a convenient Lipschitz function $\varphi_0$ for the visited verteces by the walks in $W_0$, such that $\|\varphi_0\|_{Lip}=1$. Then, we extend $\varphi_0$ to all the elements of $V$. To do this, we use a convex combination for a certain $0 \le \alpha \le 1$ of the McShane and Whitney extensions of real Lipschitz functions on metric subspaces, that is given by   
\begin{equation}\label{formula_ext}
    \widehat \varphi(v):= \alpha f^M(v) + (1-\alpha) f^W(v), \ \ v \in G,
\end{equation}
where
\begin{align*}
    f^M(v) &= \sup_{w_0\in W_0, i \in \mathbb N}\{\varphi_0(w_0(i))-  d(v,w_0(i))\} \ \ \text{and} \\
    f^W(v) &= \inf_{w_0\in W_0, i \in \mathbb N}\{\varphi_0(w_0(i))+  d(v,w_0(i))\}.
\end{align*}
This formula preserves the Lipschitz constant of $\varphi_0,$ which is equal to $1.$

\item If the proximity $P$ is known for the walks in $W_0$, we can determine the values of $(s_i^{w_1,w_2})_i$ using the formula in Remark \ref{remal}, once $\varphi_0$ is defined. Then, we find the average value $(s_i)_i=Avg((s^{w_1,w_2}_i)_i)$ of all the sequences $s^{w_1,w_2}_i,$ $w_1,w_2 \in W_0,$ $w_1 \ne w_2$. The election of the average value is up to some point arbitrary; we choose it because it could provide, when used to construct the approximation for $P$ that will be explained below, a similar value as the one of $P$ for the couples of walks in $W_0$. In another case, if we need building $P$ from scratch, we choose $(s_i)_i$ depending on the modelling process.

\item We define the proximity $\widehat P$
for all the (pieces of) walks of the graph by
$$
\widehat P(w_1,w_2,A) = \sum_{i \in A} \tau_i \, s_i \, \big| \widehat \varphi(w_1(i))-  \widehat \varphi(w_2(i)) \big|, \quad w_1,w_2 \in S_1(G), \,\, A \in \Sigma.
$$
We call it the average proximity associated with the original proximity $P,$ that was only known for $W_0.$
\end{enumerate}

The next result provides the main properties of the extended proximity, which, due to the way it has been defined, satisfies all the adequacy requirements posed when we first addressed the problem.

\begin{lemma}
The average proximity $\widehat P$ constructed as above satisfies the following properties.
\begin{itemize}

    \item[(i)] $\widehat P$ is an additive $1-$concave proximity with constant $K.$
    
    \item[(ii)] The function $\widehat P,$ considered as a distance-type function, satisfies a Lipschitz inequality with constant $\le K.$

    \item[(iii)] For every $A \in \Sigma,$ $\widehat P(\cdot,\cdot,A)$ is a pseudometric in $W(G).$
    
\end{itemize}
\end{lemma}
\begin{proof}
\begin{itemize}
\item[(i)] The first part is an obvious consequence of the fact that ``the addition of summations over disjoint subsets of indices equals the summation over their union''.
On the other hand, 
for every sequence of couples of walks $(w^1_1,w^2_1), ..., (w^1_n, w^2_n)$ and sets $A_1,...,A_n \in \Sigma,$
\begin{align*}
\sum_{k=1}^n P(w^1_k,w^2_k, A_k) &= \sum_{k=1}^n \sum_{i \in A_k} \tau_i \, s_i \, \big| \widehat \varphi(w^1_k(i))-  \widehat \varphi(w^2_k(i)) \big| \\
&\le \sup_{i} |s_i| \,
\sum_{k=1}^n \sum_{i \in A_k} \tau_i  \big| \widehat \varphi(w^1_k(i))-  \widehat \varphi(w^2_k(i)) \big| \\
&\le K \ \Big\|
 \sum_{k=1}^n \big| \langle w^1_k(A_k) \circleddash w^2_k(A_k), \cdot \rangle \big| \Big\|_{C(B)}.
\end{align*}

\item[(ii)] If $w_1,w_2 \in W(G)$ and $A \in \Sigma,$ we get the Lipschits-type inequality
\begin{align*}
\big| P(w_1,w_2,A) \big| &\le K \Big| \sum_{i \in A} \tau_i  \big| \widehat \varphi(w^1(i))-  \widehat \varphi(w^2(i)) \big| \\
&\le K \sum_{i \in A} \tau_i d(w_1(i), w_2(i)) \\
&= K d_\tau(w_1,w_2).
\end{align*}

\item[(iii)] Indeed, the formula is clearly positive and symmetric with respect to $w_1,w_2.$ On the other hand, if $A \in \Sigma$ and $w_1,w_2, w_3 \in S_1(G),$ we have that
\begin{align*}
\widehat P(w_1,w_3,A) &= \sum_{i \in A} \tau_i \, s_i \, \big| \widehat \varphi(w_1(i))-  \widehat \varphi(w_3(i)) \big| \\
&\le
\sum_{i \in A} \tau_i \, s_i \, \big| \widehat \varphi(w_1(i))-  \widehat \varphi(w_2(i)) \big| 
\\ 
& +
\sum_{i \in A} \tau_i \, s_i \, \big| \widehat \varphi(w_2(i))-  \widehat \varphi(w_3(i)) \big| \\
&= \widehat P(w_1,w_2,A) + \widehat P(w_2,w_3,A).
\end{align*}
\end{itemize}

\end{proof}

\begin{remark}
The way the extension $\widehat \varphi$ is obtained can be changed according to the characteristics of the problem. For instance, we can consider the case that $\varphi$ is known for all the vertices of the graph; in this case, no extension from the range of the walks is needed. 
\end{remark}

The obtained function can be used iteratively to construct an exploration algorithm,  together with a method to classify general walks on the graph, which improves as the number of experiments increases, that is, as knowledge about the system grows. This could be used to design a reinforcement learning algorithm to solve this class of problems.

Let us show in what follows a concrete example of how the proximity function $\widehat P$ obtained as above can be used as a tool for a classification algorithm. Let us formalize the problem we aim to address.

\vspace{0.4cm}

\textbf{Problem.} Given a set of exploratory walks $W_0$ and the associate proximity function $\widehat P,$ construct a classification rule for the set of all walks (paths) connecting $v_1$ and $v_2$ in such a way that each of the obtained groups has as central representative an element of $W_0.$

The algorithm for solving that is straightforward. Once the proximity $\widehat P$ is constructed, we only need to use it as a classifier, that is, to consider for each $w_k \in W_0,$ the set
$$
C_{w_k}= \Big\{ w \in W(G): P(w,w_k) \le P(w,w_j), \, j \ne k \Big\}.
$$
Of course, the sets $\Big\{ C_{w_k}: w_k \in W_0 \Big\}$ obtained in this way are not disjoint but, up to the limit cases that can be assigned using other criteria, the result gives a rule to obtain a partition of $W(G).$

Let us provide an illustrative example of the clustering procedure for walks on graphs. This includes, in particular, the extension step that allows us to propagate an initial evaluation (defined only on a set of explored vertices) to the entire vertex set, including nodes that have not yet been visited. From this extended evaluation, we construct a proximity function $P,$ which we will use to do a comparison (and hence clustering) of both explored and unexplored walks.
To this end, we consider the finite, undirected and connected graph $G$
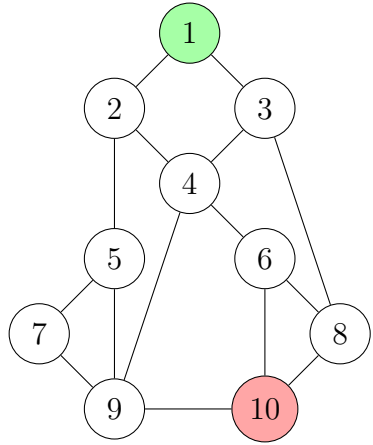
\begin{center}
\begin{tikzpicture}[
    base/.style   ={circle,draw,fill=green!35,minimum size=8mm}, 
    target/.style ={circle,draw,fill=red!35,minimum size=8mm},   
    node/.style   ={circle,draw,minimum size=8mm}                
]
  \node[base] (1) at ( 0,  2) {1};
  \node[node] (2) at (-1,  1) {2};
  \node[node] (3) at ( 1,  1) {3};
  \node[node] (4) at ( 0,  0) {4};
  \node[node] (5) at (-1, -1) {5};
  \node[node] (6) at ( 1, -1) {6};
  \node[node] (7) at (-2, -2) {7};
  \node[node] (8) at ( 2, -2) {8};
  \node[node] (9) at (-1, -3) {9};
  \node[target] (10) at ( 1, -3) {10};

  \path (1) edge (2)
        (1) edge (3)
        (2) edge (4)
        (3) edge (4)
        (2) edge (5)
        (4) edge (6)
        (4) edge (9)
        (5) edge (9)
        (6) edge (8)
        (5) edge (7)
        (7) edge (9)
        (3) edge (8)
        (6) edge (10)
        (9) edge (10)
        (8) edge (10);
\end{tikzpicture}
\end{center}
endowed with the shortest–path distance $d$.
The base vertex is $v_0=1$ (start of the walk) and the target vertex is $v_n=10$ (end of the walk).

Let us assume that only the following three walks have been explored.
\begin{figure}[h]
\centering
\begin{minipage}[t]{0.50\textwidth}
  \vspace{0pt}%
  \begin{align*}
    \color{green!60!black}{w_1} &= (v_1,v_2,v_5,v_9,v_{10},v_{10},\dots),\\
    \color{blue}{w_2} &= (v_1,v_3,v_4,v_9,v_{10},v_{10},\dots),\\
    \color{red}{w_3} &= (v_1,v_3,v_8,v_{10},v_{10},\dots).
    \end{align*}
\end{minipage}\hfill
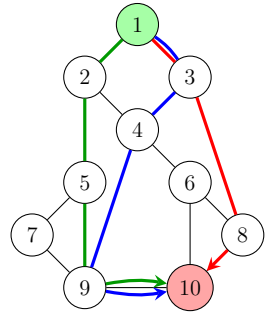
\begin{minipage}[t]{0.5\textwidth}
  \vspace{0pt}%
  \begin{center}
  \begin{tikzpicture}[scale=0.7, transform shape,
      base/.style   ={circle,draw,fill=green!35,minimum size=8mm},
      target/.style ={circle,draw,fill=red!35,minimum size=8mm},
      node/.style   ={circle,draw,minimum size=8mm},
      w1/.style     ={red, very thick, -stealth},
      w2_0/.style   ={blue, very thick},
      w2/.style     ={blue, very thick, -stealth},
      w3/.style     ={green!60!black, very thick, -stealth}
  ]
    \node[base]   (1)  at ( 0,  2) {1};
    \node[node]   (2)  at (-1,  1) {2};
    \node[node]   (3)  at ( 1,  1) {3};
    \node[node]   (4)  at ( 0,  0) {4};
    \node[node]   (5)  at (-1, -1) {5};
    \node[node]   (6)  at ( 1, -1) {6};
    \node[node]   (7)  at (-2, -2) {7};
    \node[node]   (8)  at ( 2, -2) {8};
    \node[node]   (9)  at (-1, -3) {9};
    \node[target] (10) at ( 1, -3) {10};

    \path (1) edge (2)
          (1) edge (3)
          (2) edge (4)
          (3) edge (4)
          (2) edge (5)
          (4) edge (6)
          (4) edge (9)
          (5) edge (9)
          (6) edge (8)
          (5) edge (7)
          (7) edge (9)
          (3) edge (8)
          (6) edge (10)
          (9) edge (10)
          (8) edge (10);

    \draw[w1] (1) -- (3) -- (8) -- (10);

    \draw[w2_0] (1) to[bend left=10] (3);
    \draw[w2]    (3) -- (4) -- (9)
                 to[bend right=10] (10);

    \draw[w3] (1) -- (2) -- (5) -- (9)
               to[bend left=10] (10);
  \end{tikzpicture}
  \end{center}
\end{minipage}
\end{figure}
Clearly, each sequence is eventually constant at $v_7$ and belongs to $W(G)$.

Based on the set of visited nodes, namely 
$$
V_{explored}= \{v_1,v_2,v_3,v_4,v_5,v_8,v_9,v_{10}\},
$$
we define a preliminary evaluation function$$
\varphi_{0}(v) := d(v,v_{10})-d(v_1,v_{10})=d(v,v_{10})-3,
$$
which satisfies that it belongs to $E(G)$, because $\varphi_{0}(v_1)=0$ and $\mathrm{Lip}(\varphi_{0})=1.$

The values of $\varphi_0$ on the explored nodes are as follows:
\begin{align*}
&\varphi_{0}(v_1)=\varphi_{0}(v_2)=0,\\
&\varphi_{0}(v_3)=\varphi_{0}(v_4)=\varphi_{0}(v_5)=-1,\\
&\varphi_{0}(v_8)=\varphi_{0}(v_9)=-2,\\
&\varphi_{0}(v_{10})=-3. 
\end{align*}
At this stage, the vertices $v_6$ and $v_7$ remain unexplored, and therefore the function $\varphi_0$  has not yet been extended to them.

To define an extension over the entire vertex set, we consider the mean of the McShane–Whitney formulas given by
$$
\widehat{\varphi}(v) := \frac{1}{2}\left(f_M(v)+f_W(v)\right), \ \ v \in V.
$$

Let us now evaluate this extension at the unexplored vertex $v_6.$
\begin{align*}
d(v_6,v_1)=3,\; d(v_6,v_2)=2,\; \dots \Longrightarrow 
     f_M(v_6)&=\max\{-3,-2,\dots\}=-2, \\
                     \phantom{\Longrightarrow}
     f_W(v_6)&=\min\{3,2,\dots\}=-1, \\
                     \phantom{\Longrightarrow}
     \widehat{\varphi}(v_6)&=-\frac{3}{2}.
\end{align*}
A similar computation shows that the value at $v_7$ is also $\widehat{\varphi}(v_7)=-\frac{3}{2}.$

We define the proximity function between arbitrary walks $w^1,w^2\in S_1(G)$ by
$$
P(w^1,w^2)
:=\sum_{i=1}^{\infty}\frac{1}{2^i}\,
   \bigl|\widehat{\varphi}\bigl(w^1(i)\bigr)-\widehat{\varphi}\bigl(w^2(i)\bigr)\bigr|.
$$
where we fix geometric weights $\tau_i = 2^{-i}$ and constant coefficients $s_i=1$ for all $i\geq 1.$
Consider the following unexplored walks
\begin{figure}[h]
\centering
\begin{minipage}[t]{0.50\textwidth}
  \vspace{0pt}%
  \begin{align*}
    \color{orange}{w_4} &= (v_1,v_2,v_5,v_7,v_9,v_{10},v_{10},\dots),\\
    \color{purple}{w_5} &= (v_1,v_3,v_4,v_6,v_{10},v_{10},\dots),\\
    \color{teal}{w_6} &= (v_1,v_3,v_8,v_6,v_{10},v_{10},\dots).
  \end{align*}
\end{minipage}\hfill
\begin{minipage}[t]{0.5\textwidth}
  \vspace{0pt}%
  \begin{center}
  \begin{tikzpicture}[scale=0.7, transform shape,
      base/.style   ={circle,draw,fill=green!35,minimum size=8mm},
      target/.style ={circle,draw,fill=red!35,minimum size=8mm},
      node/.style   ={circle,draw,minimum size=8mm},
      w4/.style     ={orange, very thick, -stealth},
      w5_0/.style   ={purple, very thick},
      w5/.style     ={purple, very thick, -stealth},
      w6/.style     ={teal, very thick, -stealth}
  ]
    \node[base]   (1)  at ( 0,  2) {1};
    \node[node]   (2)  at (-1,  1) {2};
    \node[node]   (3)  at ( 1,  1) {3};
    \node[node]   (4)  at ( 0,  0) {4};
    \node[node]   (5)  at (-1, -1) {5};
    \node[node]   (6)  at ( 1, -1) {6};
    \node[node]   (7)  at (-2, -2) {7};
    \node[node]   (8)  at ( 2, -2) {8};
    \node[node]   (9)  at (-1, -3) {9};
    \node[target] (10) at ( 1, -3) {10};

    \path (1) edge (2)
          (1) edge (3)
          (2) edge (4)
          (3) edge (4)
          (2) edge (5)
          (4) edge (6)
          (4) edge (9)
          (5) edge (9)
          (6) edge (8)
          (5) edge (7)
          (7) edge (9)
          (3) edge (8)
          (6) edge (10)
          (9) edge (10)
          (8) edge (10);

    \draw[w4] (1) -- (2) -- (5) -- (7) -- (9) -- (10);

    \draw[w5_0] (1) to[bend left=10] (3);
    \draw[w5]    (3) -- (4) -- (6)
                to[bend right=10] (10);

    \draw[w6]    (1) -- (3) -- (8) -- (6)
                to[bend left=10] (10);
  \end{tikzpicture}
  \end{center}
\end{minipage}
\end{figure}

Since our aim is to classify new walks based on their similarity to previously explored ones, we use the set $\{w_1,w_2,w_3\}$ as reference representatives. The proximity of each unexplored walk to these reference elements, computed via the function $P,$ can be seen in the next table.
\begin{center}
\begin{tabular}{@{}l|ccc@{}}
\toprule
$P(\cdot, \cdot)$ & $w_1$ & $w_2$ & $w_3$ \\
\midrule
$w_4$ & \color{blue}{0.062} & 0.312 & 0.5 \\
$w_5$ & 0.281 & \color{blue}{0.031} & 0.219 \\
$w_6$ & 0.406 & 0.156 & \color{blue}{0.094} \\
\bottomrule
\end{tabular}
\end{center}

\vspace{3mm}
Hence, the assignments induced by the proximity function $P$ is
$$
w_4 \sim w_1, \qquad w_5 \sim w_2, \qquad w_6 \sim w_3.
$$

\section{Potential applications: proximity-guided exploration in reinforcement learning}

As discussed in the Introduction, the results presented in this article may be applied to the development of reinforcement learning algorithms on graphs, a topic that has recently gained considerable attention (see, e.g., \cite{jac,mei,nie}). The suggested procedure involves integrating the proximity function $\widehat{P}$ as a mechanism to guide exploration on the space of walks. While the use of $\widehat{P}$ that we have explained has focused on classification and metric analysis, its geometric interpretation and ability to quantify similarity between walks suggest a broader applicability, particularly in settings where the reward landscape is sparse or the graph structure is too large for exhaustive exploration.

We propose an improvement to the exploration strategy used in reinforcement learning algorithms that incrementally construct walks within graph-based environments. Traditionally, these algorithms alternate between exploitation (choosing the next node to maximize an estimated reward) and exploration (randomly selecting a new node). The novelty lies in replacing random exploration with a proximity-guided strategy using a function $\widehat{P}$. Instead of sampling uniformly, the agent compares potential path extensions to a reference set of high-reward walks, prioritizing those that are most similar in structure. This approach introduces a more informed, data-driven method for exploration, focusing on areas of the graph that resemble previously successful trajectories.

This procedure offers several advantages. It remains model-free and is compatible with Lipschitz-based value function approximations, while also taking advantage of the geometric structure of the walk space induced by the proximity function. As the agent collects more data, it dynamically refines its understanding of similarity (via the use of a proximity), enabling adaptive and more efficient exploration. Although further research is needed to formalize the method and evaluate its performance across various graph and reward structures, proximity-based exploration could help for improving sample efficiency and scalability in RL on combinatorial and graph-structured domains.

\section*{Acknowledgements}
This research was funded by the Agencia Estatal de Investigaci\'on, grant number PID2022-138342NB-I00.

Also, by the European Union’s Horizon Europe research and innovation program under the Grant Agreement No. 
101059609 (Re-Livestock).
The second author was supported by a contract of the Programa de Ayudas de Investigación y Desarrollo (PAID-01-24), Universitat Politècnica de València.



\end{document}